\documentclass[11pt]{article}
\usepackage[top=30truemm,bottom=30truemm,left=25truemm,right=25truemm]{geometry}
\usepackage{parskip}
\usepackage[utf8]{inputenc}
\usepackage[T1]{fontenc}
\usepackage{tgtermes}
\usepackage{microtype}
\usepackage{amsmath}
\usepackage{amssymb}
\usepackage{bbm}
\usepackage{amsthm}
\usepackage{mathtools}

\usepackage[authoryear,round]{natbib}

\usepackage[ruled,vlined,linesnumbered]{algorithm2e}

\usepackage{xcolor}
\newcommand{\e}{\mathrm{e}}
\usepackage[hypertexnames=false,backref,pagebackref]{hyperref}
\hypersetup{
  colorlinks=true,
  linkcolor=blue,
  citecolor=blue,
}
\usepackage[nameinlink,capitalise,noabbrev]{cleveref}
\newcommand{\fixcrefthm}[1]{%
  \AddToHook{env/#1/begin}{\crefalias{theorem}{#1}}%
}
\fixcrefthm{theorem}
\fixcrefthm{lemma}
\fixcrefthm{proposition}
\fixcrefthm{corollary}
\fixcrefthm{assumption}
\fixcrefthm{definition}
\fixcrefthm{remark}

\crefname{equation}{}{}
\crefname{theorem}{Theorem}{Theorems}
\Crefname{theorem}{Theorem}{Theorems}
\crefname{lemma}{Lemma}{Lemmas}
\Crefname{lemma}{Lemma}{Lemmas}
\crefname{corollary}{Corollary}{Corollaries}
\Crefname{corollary}{Corollary}{Corollaries}
\crefname{definition}{Definition}{Definitions}
\Crefname{definition}{Definition}{Definitions}
\crefname{remark}{Remark}{Remarks}
\Crefname{remark}{Remark}{Remarks}
\crefname{assumption}{Assumption}{Assumptions}
\Crefname{assumption}{Assumption}{Assumptions}
\crefname{proposition}{Proposition}{Propositions}
\Crefname{proposition}{Proposition}{Propositions}
\Crefname{algocf}{Algorithm}{Algorithms}
\crefname{appendix}{Appendix}{Appendices}
\Crefname{appendix}{Appendix}{Appendices}

\SetArgSty{textnormal}
\SetKwInOut{KwRequire}{Require}
\crefname{line}{Line}{Lines}
\Crefname{line}{Line}{Lines}
\renewcommand{\nllabel}[1]{%
  \addtocounter{AlgoLine}{-1}%
  \refstepcounter{AlgoLine}%
  \label[line]{#1}%
}

\usepackage{textcase}
\usepackage{graphicx}
\graphicspath{{figure/}}
\usepackage{enumitem}
\usepackage{booktabs}
\usepackage{pdflscape}

\usepackage{wrapfig}
\usepackage{tikz}

\theoremstyle{plain}
\newtheorem{theorem}{Theorem}[section]
\newtheorem{lemma}[theorem]{Lemma}

\newtheorem{assumption}[theorem]{Assumption}
\theoremstyle{definition}

\theoremstyle{remark}

\newcommand{\calA}{\mathcal{A}}

\newcommand{\calF}{\mathcal{F}}

\newcommand{\calK}{\mathcal{K}}

\newcommand{\calM}{\mathcal{M}}

\newcommand{\calP}{\mathcal{P}}

\newcommand{\calS}{\mathcal{S}}

\renewcommand{\hat}{\widehat}
\renewcommand{\epsilon}{\varepsilon}

\def\E{\mathbb{E}}

\def\I{\mathbb{I}}
\newcommand{\ind}{\mathbbm{1}}
\def\R{\mathbb{R}}

\def\N{\mathbb{N}}

\DeclareMathOperator*{\argmin}{arg\,min}

\newcommand{\tilO}{\widetilde{O}}
\DeclarePairedDelimiter{\abs}{\lvert}{\rvert}
\DeclarePairedDelimiter{\inpr}{\langle}{\rangle}
\DeclarePairedDelimiter{\brk}{[}{]}
\DeclarePairedDelimiter{\prn}{(}{)}
\DeclarePairedDelimiter{\set}{\{}{\}}
\DeclarePairedDelimiter{\nrm}{\|}{\|}

\newcommand{\pio}{\mathring{\pi}}
\renewcommand{\hat}{\widehat}

\allowdisplaybreaks[1]
\usepackage{autonum}  %

\SetKwInput{Require}{Require}
\SetKwInput{Return}{Return}

\usepackage{tabularx,array}
\title{Closing the Horizon Gap in Policy Optimization\\
for Adversarial MDPs}
\author{
  Mingyi Li\footnote{
    The University of Tokyo;
    \texttt{mingyi-mike@g.ecc.u-tokyo.ac.jp}.
  }
  \and
  Taira Tsuchiya\footnote{
    The University of Tokyo, The University of Osaka, and RIKEN;
    \texttt{tsuchiya@mist.i.u-tokyo.ac.jp}.
  }
}

\date{October 8, 2026}

\begin{document}
\maketitle

\begin{abstract}
We consider policy optimization for online episodic tabular Markov decision processes (MDPs) with adversarial losses and bandit feedback.
Policy optimization updates the policy locally at each state and avoids optimization over the occupancy-measure polytope, but its existing regret bounds are larger by a factor of the horizon $H$ than those of occupancy-measure-based algorithms.
We close this gap by using regularized $Q$-functions, which allow us to control the stability of the local updates jointly over all state-action pairs rather than separately at each state.
The resulting algorithm attains high-probability regret bounds of $\widetilde O(\sqrt{HS(H+A)T})$ for known transitions and $\widetilde O(HS\sqrt{AT})$ for unknown transitions, where $S$ is the number of states, $A$ the number of actions, and $T$ the number of episodes.
Both bounds improve the horizon dependence of existing policy optimization bounds, and the latter matches the best-known bound.
We further extend the algorithm to adversarial linear-mixture MDPs and obtain the same improvement in the horizon dependence.
\end{abstract}

\section{Introduction}
We study policy optimization for online learning in episodic tabular Markov decision processes (MDPs) with adversarial losses and bandit feedback.
Policy optimization is widely used in reinforcement learning and has achieved empirical success (\textit{e.g.,}~\citealp{schulman2015trust,schulman2017proximal}).
In each of $T$ episodes, the environment chooses a loss function, and the learner chooses and executes a policy, observing losses only at the visited state-action pairs.
The learner aims to minimize the regret, the gap between the cumulative loss of the learner and that of the best fixed policy in hindsight.

Algorithms for online learning in episodic tabular MDPs typically use either \emph{global optimization} or \emph{policy optimization}.
Global optimization solves an optimization problem over the occupancy-measure polytope in each episode.
With this approach, \citet{zimin2013online} achieve minimax-optimal regret $\widetilde O(\sqrt{HSAT})$ for known transitions, and \citet{jin2020learning} achieve the best-known regret bound of $\widetilde O(HS\sqrt{AT})$ for unknown transitions, where $H$ is the horizon, $S$ the number of states, and $A$ the number of actions.
However, global optimization can be computationally demanding for large MDPs.

Policy optimization is a computationally efficient alternative, which updates the action distribution locally at each state~\citep{shani2020optimistic,luo2021policy}.
The best-known regret bounds for policy optimization are $\widetilde O(\sqrt{H^3SAT})$ for known transitions and $\widetilde O(H^2S\sqrt{AT})$ for unknown transitions~\citep{luo2021policy}.
These bounds have an extra factor of the horizon $H$ compared with those of global optimization.
Later work on policy optimization has extended the analysis of \citet{luo2021policy} to delayed feedback, linear transition models, and data-dependent regret bounds, but removing the extra horizon factor remains an open problem.
This leads to the following natural question.
\begin{center}
\emph{Is the extra horizon dependence intrinsic to policy optimization, or can it be improved?}
\end{center}

\subsection{Contributions of this paper}\label{sec:contributions}

We answer this question by showing that the extra horizon dependence can be improved.
We prove that policy optimization achieves high-probability regret bounds of $\widetilde O(\sqrt{HS(H+A)T})$ for known transitions (\cref{thm:known}) and $\widetilde O(HS\sqrt{AT})$ for unknown transitions (\cref{thm:unknown}).
In both settings, we improve the horizon dependence of the regret bounds of \citet{luo2021policy}.
For unknown transitions, the leading term improves by a factor of $H$ and matches the best-known regret upper bound for this setting.
\cref{tab:comparison} compares these bounds with existing results.

To obtain these bounds, we exploit regularized $Q$-functions (see \cref{eq:tildeQdef,eq:tildeVdef}) to construct and analyze our policy updates.
Replacing the ordinary $Q$-functions with regularized ones yields more optimistic $Q$-function estimates while retaining local policy updates at each state.
The regularized value at each state depends recursively on the regularized values in the next layer, and thus optimism propagates backward across layers.
Our algorithm updates the policy by exponential weights with regularized $Q$-functions and covers both known and unknown transitions.

\begin{table}[t]
\centering
\small
\setlength{\tabcolsep}{8pt}
\setlength{\belowcaptionskip}{8pt}
\setlength{\parskip}{0pt}
\renewcommand{\arraystretch}{1.2}
\caption{Regret upper bounds for episodic tabular MDPs with adversarial losses and bandit feedback under known and unknown transitions. Logarithmic factors and lower-order terms are omitted.}
\begin{tabular}{@{}lccc@{}}
\toprule
Reference & Known transitions & Unknown transitions & Policy optimization \\
\midrule
\citet{zimin2013online} & $\sqrt{HSAT}$ & --- & --- \\
\citet{jin2020learning} & --- & $HS\sqrt{AT}$ & --- \\
\citet{luo2021policy} & $\sqrt{H^3SAT}$ & $H^2S\sqrt{AT}$ & $\checkmark$ \\
\midrule
\textbf{This work} & $\sqrt{HS(H+A)T}$ & $HS\sqrt{AT}$ & $\checkmark$ \\
\bottomrule
\end{tabular}
\label{tab:comparison}
\end{table}

\begin{table}[t]
\centering
\small
\setlength{\tabcolsep}{8pt}
\setlength{\belowcaptionskip}{8pt}
\setlength{\parskip}{0pt}
\renewcommand{\arraystretch}{1.2}
\caption{Regret upper bounds for adversarial linear-mixture MDPs with bandit feedback. Logarithmic factors and lower-order terms are omitted.}
\begin{tabular}{@{}lcc@{}}
\toprule
Reference & Regret upper bound & Policy optimization \\
\midrule
\citet{lzz2024} & $\sqrt{HSAT}+dS^{3/2}\sqrt{HT}$ & --- \\
\citet{cheng2026policy} & $\sqrt{H^3SAT}+dS^{3/2}\sqrt{H^3T}$ & $\checkmark$ \\
\midrule
\textbf{This work} & $\sqrt{HSAT}+dS^{3/2}\sqrt{HT}$ & $\checkmark$ \\
\bottomrule
\end{tabular}
\label{tab:lm-comparison}
\end{table}

Regularized $Q$-functions allow us to control the stability of the algorithm jointly over all state-action pairs $(s,a)$.
Existing analyses of policy optimization first use the performance difference lemma~\citep{kakade2002approximately} to express the regret as a weighted sum of local regrets, one for each state.
They then control the stability of each local policy update separately.
Since $Q$-function estimators aggregate losses over the remaining horizon, separate stability bounds can repeatedly count the same downstream losses.
This repeated accounting leads to the extra horizon dependence in the resulting regret bounds.
In contrast, we decompose the regret jointly over all state-action pairs, as in the occupancy-measure approach.
The regularized $Q$-function connects the local policy updates across layers, allowing us to analyze their stability together.
The overall stability term then reduces to the difference between the ordinary and regularized value functions at the initial state.
We bound this difference using the expected square of the cumulative estimated loss along a trajectory.

We relate the regularized $Q$-function to existing policy optimization and global optimization methods (\cref{sec:bonus,sec:occupancy}).
The optimism of the regularized $Q$-function can be interpreted as an implicit exploration bonus in the local policy update, connecting our method to policy optimization with exploration bonuses~\citep{luo2021policy}.
We also interpret our algorithm's policy update as optimization over occupancy measures with a regularizer induced by negative conditional entropy, and thus relate it to global optimization methods~\citep{zimin2013online,jin2020learning}.

In \cref{sec:linear-mixture}, we extend the same algorithm to linear-mixture MDPs and obtain a high-probability regret bound of $\widetilde O(\sqrt{HSAT}+dS^{3/2}\sqrt{HT})$ (\cref{thm:lm}), where $d$ is the dimension of the feature map. This improves the horizon dependence of the policy optimization bound of \citet{cheng2026policy} by a factor of $H$ and matches the best-known bound of \citet{lzz2024}. \cref{tab:lm-comparison} compares this bound with existing results for linear-mixture MDPs.

\section{Related work}
\paragraph{Adversarial MDPs.}
Adversarial tabular MDPs were studied in the early work of \citet{even2009online,yu2009Markov}, and later in the episodic setting.
For known transitions and bandit feedback, \citet{zimin2013online} use global optimization and achieve regret $\widetilde O(\sqrt{HSAT})$, which matches the lower bound $\Omega(\sqrt{HSAT})$ up to logarithmic factors.
For unknown transitions, global optimization achieves regret $\widetilde O(HS\sqrt{AT})$ under full-information feedback~\citep{rosenberg2019online} and bandit feedback~\citep{jin2020learning}, leaving a $\sqrt S$ gap from the lower bound $\Omega(H\sqrt{SAT})$~\citep{jin2018q,domingues2021episodic}.

For policy optimization under bandit feedback, \citet{shani2020optimistic} establish a regret bound with $\widetilde O(T^{2/3})$ dependence on $T$.
Using dilated exploration bonuses, \citet{luo2021policy} improve this dependence to $\widetilde O(\sqrt T)$ and obtain regret bounds of $\widetilde O(\sqrt{H^3SAT})$ for known transitions and $\widetilde O(H^2S\sqrt{AT})$ for unknown transitions.
Both bounds incur an extra factor of $H$ compared with the corresponding bounds for global optimization.
Under full-information feedback and unknown transitions, this horizon gap can be closed by a policy optimization algorithm with regret $\widetilde O(HS\sqrt{AT})$~\citep{li2026policy}.
Aiming instead to improve the dependence on $S$, \citet{tiapkin2025narrowing} use policy optimization and achieve regret $\widetilde O(H^3\sqrt{SAT})$ in the same setting.

\paragraph{The horizon gap in policy optimization.}
The same horizon gap appears in the regret bounds of policy optimization in many other settings.
Under delayed bandit feedback, where losses from an episode are revealed only after a delay, global optimization achieves regret $\widetilde O(\sqrt{HSAT}+H(SA)^{1/4}\sqrt D)$ and $\widetilde O(HS\sqrt{AT}+H(SA)^{1/4}\sqrt D)$ for known and unknown transitions, respectively, where $D$ denotes the total delay~\citep{jin2022near}.
The corresponding policy optimization bounds are $\widetilde O(\sqrt{H^3SAT}+H^3\sqrt{T+D})$ and $\widetilde O(H^2S\sqrt{AT}+H^3\sqrt D)$~\citep{lancewicki2023delay}.
Under preference feedback, where the learner observes comparisons between arms rather than numerical losses, global and policy optimization achieve regret $\widetilde O((H^2S^2K)^{1/3}T^{2/3})$ and $\widetilde O((H^6SK^5)^{1/3}T^{2/3})$, respectively, for adversarial preferences with Borda-score rewards, known transitions, and $K$ arms~\citep{tsuchiya2025reinforcement}.

Under aggregate bandit feedback, the learner observes only the total loss of each trajectory rather than the losses at the visited state-action pairs.
In contrast to the settings above, policy optimization is well suited to this feedback and achieves regret $\widetilde O(\sqrt{H^3SAT}+H^3\sqrt T)$ for known transitions and $\widetilde O(H^2S\sqrt{AT})$ for unknown transitions~\citep{lancewicki2025near}.
For known transitions, this bound matches the lower bound $\Omega(\sqrt{H^3SAT})$ and the best-known bound $\widetilde O(\sqrt{H^3SAT})$ of global optimization~\citep{ito2025adapting} up to logarithmic factors when $SA\ge H^3$.

Other studies develop best-of-both-worlds policy optimization algorithms, which achieve $\mathrm{polylog}(T)$ regret for stochastic losses and $\widetilde O(\sqrt T)$ regret for adversarial losses without knowing the type of losses~\citep{dann2023best}, and some of which further achieve data-dependent regret bounds that exploit the structure of the loss sequence~\citep{li2026data,li2026policy}.
Under bandit feedback, their leading worst-case terms remain $\widetilde O(\sqrt{H^3SAT})$ for known transitions and $\widetilde O(H^2S\sqrt{AT})$ for unknown transitions, which are the same as those of \citet{luo2021policy}.
These methods therefore have the same extra factor of $H$ relative to global optimization.

Linear-mixture MDPs assume that the transition kernel is linear in a known $d$-dimensional feature map, allowing the learner to exploit this structure to generalize across states and actions~\citep{ayoub2020model,zhou2021nearly}.
For adversarial linear-mixture MDPs with unknown transitions and bandit feedback, global optimization achieves regret $\widetilde O(\sqrt{HSAT}+dS^{3/2}\sqrt{HT})$~\citep{lzz2024}.
In the same setting, policy optimization achieves regret $\widetilde O(\sqrt{H^3SAT}+dS^{3/2}\sqrt{H^3T})$~\citep{cheng2026policy}.
This bound also incurs an extra factor of $H$ in both the loss-estimation and transition-estimation terms compared with the global optimization bound.

\paragraph{Regularized $Q$-functions.}
Regularized $Q$-functions have been studied in stochastic control and reinforcement learning.
In stochastic optimal control, policy improvement with regularized $Q$-functions converges to the optimal cost~\citep{rawlik2012stochastic}.
Similarly, in reinforcement learning, a regularized $Q$-function formulation yields dynamic policy programming for finding an optimal policy in discounted MDPs~\citep{azar2012dynamic}.
Beyond these specific algorithms, a general theory of regularized MDPs analyzes how approximation errors affect policy performance~\citep{geist2019theory}.
These regularized policy updates are also connected to occupancy measures.
\citet{neu2017unified} relate conditional-entropy regularization of occupancy measures to policy updates and interpret the finite-horizon update of \citet{rawlik2012stochastic} as mirror descent.
Unlike these works, we consider adversarial losses and use regularized $Q$-functions to analyze the stability of policy optimization under bandit feedback in (non-regularized) episodic MDPs.

\section{Preliminaries}\label{sec:prelim}
We consider an episodic tabular Markov decision process (MDP) $\calM=(\calS,\calA,P,H,s_0)$, where $\calS$ is a finite state space with $S=\abs{\calS}$ states, $\calA$ is a finite action space with $A=\abs{\calA}$ actions, $H$ is the episode horizon, and $s_0$ is the initial state. The transition kernel $P$ is fixed but may be unknown to the learner, where $P(s'\mid s,a)$ is the probability of reaching state $s'$ after taking action $a$ in state $s$.
Following the standard layered formulation~\citep{neu2010online,jin2020learning,luo2021policy}, we assume that $\calS$ is partitioned into $H+1$ nonempty disjoint layers $\calS_0,\ldots,\calS_H$, with initial layer $\calS_0=\{s_0\}$ and terminal absorbing layer $\calS_H=\{s_H\}$. For $s\in\calS_h$, let $h(s)=h$. Transitions only connect consecutive layers: $P(\cdot\mid s,a)$ is a probability distribution on $\calS_{h+1}$ for every $(s,a)\in\calS_h\times\calA$. For convenience, we exclude the terminal layer $\calS_H$ from $\calS$ in what follows, so that $H\leq S$.

A policy is a mapping $\pi:\calS\to\Delta(\calA)$, where $\Delta(\calK)$ denotes the set of probability distributions over a finite set $\calK$. The learner interacts with the MDP for $T$ episodes. At the beginning of episode $t$, the adversarial environment selects a loss function $\ell_t:\calS\times\calA\to[0,1]$, which may depend on the learner's algorithm and past history, but not on the randomness of the current episode. The learner then chooses $\pi_t$ based on past observations, executes it under $P$ from $s_{t,0}=s_0$, and observes only the visited state-action pairs and their losses $\{(s_{t,h},a_{t,h},\ell_t(s_{t,h},a_{t,h}))\}_{h=0}^{H-1}$.

For a policy $\pi$ and a function $f:\calS\times\calA\to\R$, define recursively the value function $V^\pi(s;f)$ and the $Q$-function $Q^\pi(s,a;f)$ under the transition kernel $P$ as $V^\pi(s;f)=\E_{a\sim\pi(\cdot\mid s)}\brk*{Q^\pi(s,a;f)}$ and $Q^\pi(s,a;f)=f(s,a)+\E_{s'\sim P(\cdot\mid s,a)}\brk*{V^\pi(s';f)}$, with $V^\pi(s_H;f)=0$.
For a policy $\pi$, let $q^{\pi}(s)$ and $q^{\pi}(s,a)$ denote the probabilities of visiting $s$ and $(s,a)$ under the transition kernel $P$ from $s_0$, respectively, and let $q^{\pi}(s'\mid s,a)$ and $q^{\pi}(s',a'\mid s,a)$ denote the visitation probabilities starting from $(s,a)$, which are zero when $h(s')<h(s)$.
The goal of the learner is to minimize the regret
\begin{align}\label{eq:regret}
\operatorname{Reg}_T\coloneqq\sum_{t=1}^T V^{\pi_t}(s_0;\ell_t)-\min_{\pi}\sum_{t=1}^T V^{\pi}(s_0;\ell_t),
\end{align}
and let $\pio\in\argmin_{\pi}\sum_{t=1}^T V^{\pi}(s_0;\ell_t)$ be an optimal fixed comparator policy in hindsight.

\paragraph{Additional notation.}
For $N\in\N$, let $[N]=\{1,\ldots,N\}$. We denote the indicator of an event $E$ by $\ind\{E\}$ and use $\log$ for the natural logarithm. For $u,p\in\Delta(\calA)$, we define their Kullback--Leibler divergence as $\operatorname{KL}(u\parallel p)=\sum_{a\in\calA}u(a)\log(u(a)/p(a))$, with the conventions $0\log(0/p)=0$ and $u\log(u/0)=+\infty$ for $u>0$. For nonnegative functions $f$ and $g$, we write $f\lesssim g$ or $f=O(g)$ if $f\leq Cg$ for some universal constant $C>0$, and use $\tilO(\cdot)$ to hide polylogarithmic factors.
Let $\calF_t$ be the sigma-algebra generated by the history up to the end of episode $t-1$, and write $\E_t[\cdot]=\E[\cdot\mid\calF_t]$. Define $\I_t(s,a)=\ind\{(s_{t,h},a_{t,h})=(s,a),\exists h\in\{0,\ldots,H-1\}\}$ to indicate whether $(s,a)$ is visited in episode $t$, with $\I_t(s)=\sum_{a\in\calA}\I_t(s,a)$.

\subsection{Confidence sets}\label{sec:confidence}
We consider both the setting in which the transition kernel $P$ is known to the learner and the setting in which it is unknown.
When $P$ is unknown, we use confidence sets for $P$, following \citet{jin2020learning}. Let $n_t(s,a)$ and $n_t(s,a,s')$ count visits to $(s,a)$ and transitions from $(s,a)$ to $s'$ before episode $t$. Define $\hat P_t(s'\mid s,a)=n_t(s,a,s')/n_t(s,a)$ when $n_t(s,a)>0$, and let $\hat P_t(\cdot\mid s,a)$ be uniform on $\calS_{h+1}$ when $n_t(s,a)=0$.
For $\delta\in(0,1)$, define the confidence set
\begin{align}
\calP_t&=\Bigl\{\widetilde P:
\abs*{\widetilde P(s'\mid s,a)-\hat P_t(s'\mid s,a)}\leq \min\biggl\{2\sqrt{\frac{\hat P_t(s'\mid s,a)\log\prn*{\frac{SAT}{\delta}}}{\max\{1,n_t(s,a)\}}}+\frac{14\log\prn*{\frac{SAT}{\delta}}}{3\max\{1,n_t(s,a)\}},1\biggr\},\\
&\qquad\forall h, \forall (s,a,s')\in\calS_h\times\calA\times\calS_{h+1}, \widetilde P(\cdot\mid s,a) \in \Delta(\calS_{h+1})\Bigr\}. \label{eq:confrow}
\end{align}
By \citet[Lemma~2]{jin2020learning}, $P\in\calP_t$ for all $t\in[T]$ with probability at least $1-4\delta$.
The constraints in \eqref{eq:confrow} are imposed on each $(s,a)$ separately, and thus $\calP_t$ is rectangular. We write $\calP_t(s,a)\subseteq\Delta(\calS_{h+1})$ for the confidence set of $P(\cdot\mid s,a)$, and $\calP_t=\prod_{s,a}\calP_t(s,a)$.
When $P$ is known, we set $\calP_t=\{P\}$.
For a transition kernel $\widetilde P$ and a policy $\pi$, let $q^{\widetilde P,\pi}(s)$ and $q^{\widetilde P,\pi}(s,a)$ denote the probabilities of visiting $s$ and $(s,a)$. Define the upper and lower occupancies by
\begin{align}\label{eq:occupancy}
\overline q_t^\pi(s)=\max_{\widetilde P\in\calP_t}q^{\widetilde P,\pi}(s),\qquad
\underline q_t^\pi(s)=\min_{\widetilde P\in\calP_t}q^{\widetilde P,\pi}(s),
\end{align}
and set $\overline q_t^\pi(s,a)=\overline q_t^\pi(s)\pi(a\mid s)$ and $\underline q_t^\pi(s,a)=\underline q_t^\pi(s)\pi(a\mid s)$. 
The upper and lower occupancies can be computed efficiently by dynamic programming over the layers~\citep[Algorithm~3]{jin2020learning}.
Similarly, $Q^{\widetilde P,\pi}(s,a;\cdot)$ and $V^{\widetilde P,\pi}(s;\cdot)$ denote the corresponding functions under a transition kernel $\widetilde P$, and the superscript $\widetilde P$ is omitted when $\widetilde P=P$.

\section{Regularized $Q$-functions}\label{sec:technical}
In this section, we define the regularized $Q$-function used in our algorithm and explain its connections to exploration bonuses and global optimization over occupancy measures.
Existing policy optimization algorithms estimate the $Q$-function and then subtract an exploration bonus from the estimate to make it optimistic~\citep{luo2021policy,dann2023best,li2026policy}.
We instead incorporate this correction directly into the $Q$-function itself and update the policy with the regularized $Q$-function.

For a policy $\pi$, a function $f:\calS\times\calA\to\R$, and a learning rate $\eta>0$, we define the \emph{regularized $Q$-function} $\widetilde Q^{\pi}(s,a;f)$ and the \emph{regularized value function} $\widetilde V^{\pi}(s;f)$ under the transition kernel $P$. With $\widetilde V^{\pi}(s_H;f)=0$, we define recursively
\begin{align}
\widetilde Q^{\pi}(s,a;f)&=f(s,a)+\E_{s'\sim P(\cdot\mid s,a)}\brk*{\widetilde V^{\pi}(s';f)},\label{eq:tildeQdef}\\
\widetilde V^{\pi}(s;f)&=-\frac1\eta\log\E_{a\sim\pi(\cdot\mid s)}\brk*{\e^{-\eta\widetilde Q^{\pi}(s,a;f)}}.\label{eq:tildeVdef}
\end{align}
We use $\widetilde Q^{\widetilde P,\pi}$ and $\widetilde V^{\widetilde P,\pi}$ for the corresponding functions under a transition kernel $\widetilde P$, and we omit the superscript $\widetilde P$ when $\widetilde P=P$.
The regularized value function can also be expressed as
\begin{align}
\widetilde V^{\pi}(s;f)
&=\min_{p\in\Delta(\calA)}\biggl\{
\E_{a\sim p}\brk*{\widetilde Q^{\pi}(s,a;f)}+\frac1\eta\operatorname{KL}\prn*{p\parallel\pi(\cdot\mid s)}\biggr\},\label{eq:gibbs}
\end{align}
with minimizer $p(a)\propto\pi(a\mid s)\e^{-\eta\widetilde Q^{\pi}(s,a;f)}$.
Here, $\pi$ is the reference policy for the KL regularizer.
Our algorithm updates the policy at each state $s$ by choosing the minimizer in \cref{eq:gibbs}, namely,
\begin{align}\label{eq:expupdate}
\pi^+(a\mid s)\propto\pi(a\mid s)\e^{-\eta\widetilde Q^{\pi}(s,a;f)}.
\end{align}
\subsection{Relation to exploration bonuses}\label{sec:bonus}
We now show that the regularized $Q$-function implicitly incorporates an exploration bonus introduced by \citet{luo2021policy}. Taking $p=\pi(\cdot\mid s)$ in \cref{eq:gibbs} gives $\widetilde V^{\pi}(s;f)\leq\E_{a\sim\pi(\cdot\mid s)}\brk[\big]{\widetilde Q^{\pi}(s,a;f)}$, and we define the statewise gap $g(s)$ by
$g(s)=\E_{a\sim\pi(\cdot\mid s)}\brk*{\widetilde Q^{\pi}(s,a;f)}-\widetilde V^{\pi}(s;f)\geq0$,
and set $g(s,a)=g(s)$ for all $a\in\calA$. Then, the regularized value function is related to the ordinary value function by $\widetilde V^{\pi}(s;f)=V^{\pi}(s;f-g)$ for every $s\in\calS$. We show this identity by backward induction from $s_H$, where both sides are zero. For every $s\in\calS_h$,
\begin{align}
\widetilde V^{\pi}(s;f)
&=\E_{a\sim\pi(\cdot\mid s)}\brk*{\widetilde Q^{\pi}(s,a;f)}-g(s)\\
&=\E_{a\sim\pi(\cdot\mid s)}\Bigl[f(s,a)-g(s)+\E_{s'\sim P(\cdot\mid s,a)}\brk*{V^{\pi}(s';f-g)}\Bigr]\\
&=V^{\pi}(s;f-g),\label{eq:softvalue}
\end{align}
where the second equality follows from the induction hypothesis on layer $h+1$. Substituting \cref{eq:softvalue} into \cref{eq:tildeQdef} also gives $\widetilde Q^{\pi}(s,a;f)-g(s)=Q^{\pi}(s,a;f-g)$.

We now show that this identity yields an exploration bonus in the update \cref{eq:expupdate}. Since $g(s)$ does not depend on the action, it cancels in the normalization of \cref{eq:expupdate}, and hence
\begin{align}
\pi^+(a\mid s)
&\propto\pi(a\mid s)\e^{-\eta Q^{\pi}(s,a;f-g)}\\
&=\pi(a\mid s)\e^{-\eta\prn*{Q^{\pi}(s,a;f)-Q^{\pi}(s,a;g)}}.\label{eq:bonusupdate}
\end{align}
The right-hand side of \cref{eq:bonusupdate} is the update of \citet{luo2021policy} with the exploration bonus $Q^{\pi}(s,a;g)$. Hence the regularized $Q$-function incorporates the exploration bonus implicitly.

\subsection{Relation to global optimization}\label{sec:occupancy}
We now show that the local policy updates in \cref{eq:expupdate} together solve a global optimization problem over occupancy measures.
The algorithm itself remains a policy optimization method, which updates the policy locally at each state by \cref{eq:expupdate}. The interpretation below explains why the regret of these local updates can nevertheless be analyzed jointly over all state-action pairs, as in the occupancy-measure approach~\citep{zimin2013online,jin2020learning}.

Let $\Omega(P)$ be the set of all occupancy measures under $P$.
The Bregman divergence $D_\Psi$ induced by the negative conditional entropy $\Psi(q)=\sum_{s,a}q(s,a)\log(q(s,a)/q(s))$ satisfies
\begin{align}\label{eq:bregmankl}
D_\Psi(q^{\pi'},q^{\pi})=\sum_s q^{\pi'}(s)\operatorname{KL}\prn*{\pi'(\cdot\mid s)\parallel\pi(\cdot\mid s)}
\end{align}
for all policies $\pi'$ and $\pi$.
Since $\pi^+$ in \cref{eq:expupdate} is the minimizer of~\cref{eq:gibbs} at every state, applying \cref{eq:gibbs} backward from layer $H-1$ to layer $0$ gives
\begin{align}\label{eq:globalgibbs}
q^{\pi^+}\in\argmin_{q\in\Omega(P)}\set*{\inpr{q,f}+\frac1\eta D_\Psi(q,q^{\pi})},
\end{align}
with optimal value $\widetilde V^{\pi}(s_0;f)$.
Thus the local updates \cref{eq:expupdate} at all states together solve a KL-regularized problem over occupancy measures. This property lets the analysis in \cref{sec:known} treat all state-action pairs jointly, while the algorithm only solves an optimization over $\Delta(\calA)$ at each state.

\section{Main algorithm}\label{sec:algorithm}
\begin{algorithm*}[t]
\caption{Policy optimization with regularized $Q$-functions}\label{alg:po}
\DontPrintSemicolon
\KwRequire{Number of episodes $T$, confidence parameter $\delta\in(0,1)$, learning rate $\eta>0$, exploration rate $\gamma>0$, and initial policy $\pi_1(a\mid s)=1/A$ for all $(s,a)$.}
\For{$t=1,\ldots,T$}{
  Construct $\calP_t$ by \cref{eq:confrow}, or $\calP_t=\{P\}$ when $P$ is known.\nllabel{line:confset}\;
  Execute $\pi_t$ and observe the trajectory $\{(s_{t,h},a_{t,h},\ell_t(s_{t,h},a_{t,h}))\}_{h=0}^{H-1}$.\nllabel{line:execute}\;
  Construct the loss estimator $\displaystyle\hat\ell_t(s,a)=\frac{\I_t(s,a)\ell_t(s,a)}{\overline q_t^{\pi_t}(s,a)+\gamma}$ for all $(s,a)$.\nllabel{line:lossest}\;
  Set $\hat V_t(s_H)=0$.\nllabel{line:terminal}\;
  \For{$h=H-1,\ldots,0$\nllabel{line:backward}}{
    Compute $\displaystyle\hat Q_t(s,a)=\hat\ell_t(s,a)+\min_{p\in\calP_t(s,a)}\sum_{s'\in\calS_{h+1}}p(s')\hat V_t(s')$ for all $(s,a)\in\calS_h\times\calA$.\nllabel{line:qhat}\;
    Compute $\displaystyle\hat V_t(s)=-\frac1\eta\log\sum_{a\in\calA}\pi_t(a\mid s)\e^{-\eta\hat Q_t(s,a)}$ for all $s\in\calS_h$.\nllabel{line:vhat}\;
  }
  Update $\pi_{t+1}(a\mid s)=\pi_t(a\mid s)\e^{-\eta(\hat Q_t(s,a)-\hat V_t(s))}$ for all $(s,a)$.\nllabel{line:update}\;
}
\end{algorithm*}

This section presents the policy optimization algorithm summarized in \cref{alg:po}, using the regularized $Q$-function introduced in the last section. 
\subsection{Algorithm design}\label{sec:components}
In episode $t$, we execute $\pi_t$ and observe the trajectory $\{(s_{t,h},a_{t,h},\ell_t(s_{t,h},a_{t,h}))\}_{h=0}^{H-1}$ (\cref{line:execute}). To update the policy using these observations, existing policy optimization algorithms estimate the ordinary $Q$-function directly from bandit feedback and subtract an exploration bonus from the estimate~\citep{luo2021policy,dann2023best,li2026policy}. We instead first estimate the losses using the estimator of \citet{jin2020learning} (\cref{line:lossest}),
\begin{align}\label{eq:unknownest}
\hat\ell_t(s,a)=\frac{\I_t(s,a)\ell_t(s,a)}{\overline q_t^{\pi_t}(s,a)+\gamma},
\end{align}
where $\gamma>0$ is the exploration rate and $\overline q_t^{\pi_t}$ is the upper occupancy.
We then construct an optimistic $Q$-function estimate by choosing the transition kernel in $\calP_t$ that minimizes the regularized $Q$-function, with $\hat\ell_t$ as the input and $\pi_t$ as the reference policy (\crefrange{line:terminal}{line:vhat}),
\begin{align}\label{eq:optimistictransition}
\hat Q_t(s,a)
=\hat\ell_t(s,a)+\min_{\widetilde P\in\calP_t}
\E_{s'\sim\widetilde P(\cdot\mid s,a)}\brk*{\widetilde V^{\widetilde P,\pi_t}(s';\hat\ell_t)}.
\end{align}
Since $\calP_t$ is rectangular, there is a transition kernel in $\calP_t$ that attains the minima in \cref{eq:optimistictransition} for all $(s,a)$ simultaneously, and we write this kernel as $\underline P_t$, so that $\hat Q_t(s,a)=\widetilde Q^{\underline P_t,\pi_t}(s,a;\hat\ell_t)$ and $\hat V_t(s)=\widetilde V^{\underline P_t,\pi_t}(s;\hat\ell_t)$.
For $s\in\calS_h$, computing $\hat Q_t(s,a)$ only requires $\hat V_t$ on layer $h+1$, so $\hat Q_t$ and $\underline P_t$ are computed backward from $h=H-1$ to $h=0$ ({\crefrange{line:terminal}{line:vhat}}).
We then update the policy at every state $s$ by exponential weights with the losses $\hat Q_t(s,\cdot)$, which gives $\pi_{t+1}(a\mid s)\propto\pi_t(a\mid s)\e^{-\eta\hat Q_t(s,a)}$.
Since the value $\hat V_t(s)$ computed in \cref{line:vhat} satisfies $\sum_b\pi_t(b\mid s)\e^{-\eta\hat Q_t(s,b)}=\e^{-\eta\hat V_t(s)}$ by \cref{eq:tildeVdef}, this update is written as in \cref{line:update}.

\subsection{Regret upper bounds}\label{sec:guarantees}
\cref{alg:po} achieves the following regret upper bounds under known and unknown transitions.

\begin{theorem}\label{thm:known}
Under known transitions, \cref{alg:po} with $\gamma=\eta=\sqrt{\frac{H\log(SA/\delta)}{S(H+A)T}}$ achieves, with probability at least $1-3\delta$,
\begin{align}\label{eq:knownrate}
\operatorname{Reg}_T\lesssim\sqrt{HS(H+A)T\log\prn*{\frac{SA}{\delta}}} + H^2\log\prn*{\frac{SA}{\delta}}.
\end{align}
\end{theorem}

\begin{theorem}\label{thm:unknown}
Under unknown transitions, \cref{alg:po} with the confidence sets \eqref{eq:confrow} and $\gamma=\eta=\sqrt{\frac{H\log(SAT/\delta)}{S(H+A)T}}$ achieves, with probability at least $1-9\delta$,
\begin{align}\label{eq:unknownrate}
\operatorname{Reg}_T\lesssim
HS\sqrt{AT\log\prn*{\frac{SAT}{\delta}}}+S^3A\log^2\prn*{\frac{SAT}{\delta}}.
\end{align}
\end{theorem}

In both settings, these bounds improve the horizon dependence of the policy optimization bounds of \citet{luo2021policy}.
For known transitions, when $H\lesssim A$, the leading term becomes $\sqrt{HSAT\log(SA/\delta)}$, which matches the lower bound $\Omega(\sqrt{HSAT})$ of \citet{zimin2013online} up to logarithmic factors.
For unknown transitions, the leading term matches the best-known bound $\widetilde O(HS\sqrt{AT})$ of \citet{jin2020learning}, obtained by global optimization over occupancy measures.
Note that the additive term $S^3A\log^2(SAT/\delta)$ also arises in \citet{jin2020learning} with the same order when their lower-order terms are made explicit.

\subsection{Regret analysis}\label{sec:known}
We first decompose the regret $\operatorname{Reg}_T=\sum_{t=1}^T\inpr{q^{\pi_t}-q^{\pio},\ell_t}$ into the estimation errors of the losses and the regret on the estimated losses:
\begin{align}
\operatorname{Reg}_T
&=\underbrace{\sum_{t=1}^T\prn*{\inpr{q^{\pi_t},\ell_t}
-\inpr{\underline q_t^{\pi_t},\hat\ell_t}}}_{\textup{\textsc{Bias}}_1}+\underbrace{\sum_{t=1}^T\inpr{\underline q_t^{\pi_t}-q^{\pio},\hat\ell_t}}_{\textup{\textsc{Reg}}}
+\underbrace{\sum_{t=1}^T\inpr{q^{\pio},\hat\ell_t-\ell_t}}_{\textup{\textsc{Bias}}_2}.
\label{eq:unknowndecomp}
\end{align}
The term $\textup{\textsc{Bias}}_1$ is the estimation error of the learner, $\textup{\textsc{Reg}}$ is the regret on the estimated losses, and $\textup{\textsc{Bias}}_2$ is the estimation error of the comparator. 
This decomposition follows the occupancy-measure approach of \citet{jin2020learning}, treating all $(s,a)$ jointly rather than using the performance difference lemma to split the regret by state as in {\citet{shani2020optimistic}} and \citet{luo2021policy}.

In what follows, we will upper bound the three terms in \cref{eq:unknowndecomp}. 
The upper bounds on the two bias terms follow by the standard arguments of \citet[Lemmas~6 and~14]{jin2020learning} (see \cref{app:unknownregret} for proofs).

\begin{lemma}[Bound on $\textup{\textsc{Bias}}_1$]\label{lem:unknownbias1}
Suppose that $P\in\calP_t$ for all $t\in[T]$. Then, with probability at least $1-\delta$,
\begin{align}
\textup{\textsc{Bias}}_1
&\leq\sum_{t=1}^T\sum_{s,a}
\prn*{\overline q_t^{\pi_t}(s,a)-\underline q_t^{\pi_t}(s,a)}+\gamma SAT+H\sqrt{2T\log\prn*{\frac{1}{\delta}}}.
\end{align}
\end{lemma}
The occupancy width $\sum_{t=1}^T\sum_{s,a}
\prn[\big]{\overline q_t^{\pi_t}(s,a)-\underline q_t^{\pi_t}(s,a)}$ appearing in the bound on $\textup{\textsc{Bias}}_1$ satisfies the following bound.
\begin{lemma}[{Occupancy width, adapted from \citealt[Lemma~E.5]{lancewicki2022cooperative}, see also \citealt[Lemma~4]{jin2020learning}}]\label{lem:widthappendix}
Suppose that $P\in\calP_t$ for all $t\in[T]$. Then, with probability at least $1-2\delta$,
\begin{align}\label{eq:widthstated}
&\sum_{t=1}^T\sum_{s,a}
\prn*{\overline q_t^{\pi_t}(s,a)-\underline q_t^{\pi_t}(s,a)}
\lesssim HS\sqrt{AT\log\prn*{\frac{SAT}{\delta}}}
+S^3A\log^2\prn*{\frac{SAT}{\delta}}.
\end{align}
\end{lemma}

\begin{lemma}[Bound on $\textup{\textsc{Bias}}_2$]\label{lem:unknownbias2}
Suppose that $P\in\calP_t$ for all $t\in[T]$. Then, with probability at least $1-\delta$,
\begin{align}
\textup{\textsc{Bias}}_2\leq\frac{H}{2\gamma}\log\prn*{\frac{SA}{\delta}}.
\end{align}
\end{lemma}

We now bound the remaining term, $\textup{\textsc{Reg}}$.

\begin{lemma}[Bound on $\textup{\textsc{Reg}}$]\label{lem:unknownreg}
Suppose that $P\in\calP_t$ for all $t\in[T]$. Then, with probability at least $1-\delta$,
\begin{align}
\textup{\textsc{Reg}}
&\leq\frac{H\log A}{\eta}
+\frac{\eta ST(A+2H)}{2}+\frac{3\eta H^2}{4\gamma}\log\prn*{\frac{H}{\delta}}.
\end{align}
\end{lemma}

\begin{proof}[Proof of \cref{lem:unknownreg}]
Since $\hat\ell_t\geq0$ and $\underline q_t^{\pi_t}\leq q^{\underline P_t,\pi_t}$, we have
$\inpr{\underline q_t^{\pi_t},\hat\ell_t}\leq\inpr{q^{\underline P_t,\pi_t},\hat\ell_t}=V^{\underline P_t,\pi_t}(s_0;\hat\ell_t)$.
Hence
\begin{align}
\textup{\textsc{Reg}}
&\leq\underbrace{\sum_{t=1}^T\prn*{\hat V_t(s_0)
-\inpr{q^{\pio},\hat\ell_t}}}_{\textup{\textsc{Penalty}}}
+\underbrace{\sum_{t=1}^T\prn*{V^{\underline P_t,\pi_t}
(s_0;\hat\ell_t)-\hat V_t(s_0)}}_{\textup{\textsc{Stability}}}.
\label{eq:penaltystability}
\end{align}
We first bound $\textup{\textsc{Penalty}}$. Substituting $\hat\ell_t(s,a)=\hat Q_t(s,a)-\sum_{s'}\underline P_t(s'\mid s,a)\hat V_t(s')$ gives
\begin{align}
&\hat V_t(s_0)-\inpr{q^{\pio},\hat\ell_t}\\
&=\hat V_t(s_0)-\sum_{s,a}q^{\pio}(s,a)\hat Q_t(s,a)+\sum_{s,a,s'}q^{\pio}(s,a)\underline P_t(s'\mid s,a)\hat V_t(s')\\
&=\hat V_t(s_0)-\sum_{s,a}q^{\pio}(s,a)\hat Q_t(s,a)+\sum_{s'\neq s_0}q^{\pio}(s')\hat V_t(s')-\sum_{s,a,s'}q^{\pio}(s,a)
\prn*{P(s'\mid s,a)-\underline P_t(s'\mid s,a)}\hat V_t(s')\\
&\leq\sum_s q^{\pio}(s)\prn*{\hat V_t(s)
-\sum_a\pio(a\mid s)\hat Q_t(s,a)},\label{eq:penaltyflow}
\end{align}
where the inequality follows because $\underline P_t$ attains the minimum in \cref{eq:optimistictransition} and $P\in\calP_t$.
Substituting $\hat V_t(s)-\hat Q_t(s,a)=\eta^{-1}\log(\pi_{t+1}(a\mid s)/\pi_t(a\mid s))$ from \cref{line:update} into \cref{eq:penaltyflow} and summing over $t$ give
\begin{align}
\textup{\textsc{Penalty}}
&=\sum_{t=1}^T\prn*{\hat V_t(s_0)-\inpr{q^{\pio},\hat\ell_t}}\\*
&\leq\frac1\eta\sum_{t=1}^T\sum_{s,a}q^{\pio}(s)\pio(a\mid s)
\log\prn*{\frac{\pi_{t+1}(a\mid s)}{\pi_t(a\mid s)}}\\*
&=\frac1\eta\sum_{t=1}^T\bigl(D_\Psi(q^{\pio},q^{\pi_t})-D_\Psi(q^{\pio},q^{\pi_{t+1}})\bigr)\\*
&\leq\frac1\eta D_\Psi(q^{\pio},q^{\pi_1})\leq\frac{H\log A}{\eta},\label{eq:penaltybound}
\end{align}
where the second equality follows from \cref{eq:bregmankl}, and the last inequality from the uniformity of $\pi_1$.

For $\textup{\textsc{Stability}}$, let $(S_h,A_h)_{h=0}^{H-1}$ be an auxiliary trajectory under $(\underline P_t,\pi_t)$ and let $Z_t=\sum_{h=0}^{H-1}\hat\ell_t(S_h,A_h)$ be its random cumulative estimated loss. Then $V^{\underline P_t,\pi_t}(s_0;\hat\ell_t)=\E_{\underline P_t,\pi_t}[Z_t]$. We first show $\hat V_t(s_0)\geq-\eta^{-1}\log\E_{\underline P_t,\pi_t}[\e^{-\eta Z_t}]$. The definitions of $\hat Q_t,\hat V_t$ and Jensen's inequality give, for every state $s$,
\begin{align}
\e^{-\eta\hat V_t(s)}
&\leq\E_{\substack{a\sim\pi_t(\cdot\mid s),s'\sim\underline P_t(\cdot\mid s,a)}}
\brk*{\e^{-\eta\hat\ell_t(s,a)-\eta\hat V_t(s')}}.
\end{align}
Repeatedly applying this inequality and using $\hat V_t(s_H)=0$ gives $\e^{-\eta\hat V_t(s_0)}\leq\E_{\underline P_t,\pi_t}[\e^{-\eta Z_t}]$. Taking logarithms thus gives $\hat V_t(s_0)\geq-\eta^{-1}\log\E_{\underline P_t,\pi_t}[\e^{-\eta Z_t}]$.

Using $Z_t\geq0$, $\log x\leq x-1$, and $\e^{-x}\leq1-x+x^2/2$ for $x\geq0$, we obtain
\begin{align}
V^{\underline P_t,\pi_t}(s_0;\hat\ell_t)-\hat V_t(s_0)
&\leq\E_{\underline P_t,\pi_t}\brk*{Z_t}+\frac1\eta\log\E_{\underline P_t,\pi_t}\brk*{\e^{-\eta Z_t}}\\
&\leq\frac1\eta\E_{\underline P_t,\pi_t}
\brk*{\eta Z_t+\e^{-\eta Z_t}-1}\leq\frac{\eta}{2}\E_{\underline P_t,\pi_t}\brk*{Z_t^2}.\label{eq:secondmoment}
\end{align}
Expanding $Z_t^2$ gives
\begin{align}
\E_{\underline P_t,\pi_t}\brk*{Z_t^2}&=\sum_{s,a}q^{\underline P_t,\pi_t}(s,a)\hat\ell_t(s,a)^2\\
&\qquad+2\sum_{0\leq h<k<H}\sum_{(s,a)\in\calS_h\times\calA}
q^{\underline P_t,\pi_t}(s,a)\hat\ell_t(s,a)\sum_{(s',a')\in\calS_k\times\calA}
q^{\underline P_t,\pi_t}(s',a'\mid s,a)\hat\ell_t(s',a')\\
&\leq\sum_{s,a}\hat\ell_t(s,a)+2\sum_{k=1}^{H-1}k\sum_{(s,a)\in\calS_k\times\calA}
\pi_t(a\mid s)\hat\ell_t(s,a)\\
&\leq\sum_{s,a}\hat\ell_t(s,a)
+2H\sum_{s,a}\pi_t(a\mid s)\hat\ell_t(s,a)\\
&=\sum_{s,a}\prn*{1+2H\pi_t(a\mid s)}\hat\ell_t(s,a),
\label{eq:trajectorymoment}
\end{align}
where the first inequality follows from 
$q^{\underline P_t,\pi_t}(s,a)\hat\ell_t(s,a)=q^{\underline P_t,\pi_t}(s,a)\I_t(s,a)\ell_t(s,a)/(\overline q_t^{\pi_t}(s,a)+\gamma)\leq\I_t(s,a)\leq1$
with $\sum_{(s,a)\in\calS_h\times\calA}\I_t(s,a)\leq1$ and $q^{\underline P_t,\pi_t}(s',a'\mid s,a)\leq\pi_t(a'\mid s')$.

Summing \cref{eq:secondmoment,eq:trajectorymoment} over $t$ gives
\begin{align}
\textup{\textsc{Stability}}
&\leq\frac\eta2\sum_{t=1}^T\sum_{s,a}
(1+2H\pi_t(a\mid s))\ell_t(s,a)+\frac\eta2\sum_{t=1}^T\sum_{s,a}(1+2H\pi_t(a\mid s))\prn*{\hat\ell_t(s,a)-\ell_t(s,a)}.
\end{align}
The first term is at most $\eta ST(A+2H)/2$ since $\ell_t\leq1$ and $\sum_a\pi_t(a\mid s)=1$.
For the second term, the concentration bound obtained by applying \cref{lem:jointIX} in \cref{app:concentration} with $\alpha_t(s,a)=2\gamma(1+2H\pi_t(a\mid s))/(2H+1)$, together with $q^{\pi_t}(s,a)\leq\overline q_t^{\pi_t}(s,a)$, gives, with probability at least $1-\delta$,
\begin{align}\label{eq:weightederror}
\sum_{t=1}^T\sum_{s,a}(1+2H\pi_t(a\mid s))
\prn*{\hat\ell_t(s,a)-\ell_t(s,a)}
\leq\frac{2H^2+H}{2\gamma}\log\prn*{\frac H\delta}
\leq\frac{3H^2}{2\gamma}\log\prn*{\frac H\delta}.
\end{align}
These two bounds yield
\begin{align}
\textup{\textsc{Stability}}
&\leq\frac{\eta ST(A+2H)}{2}+\frac{3\eta H^2}{4\gamma}\log\prn*{\frac{H}{\delta}}. \label{eq:stabilitysum}
\end{align}
Combining \cref{eq:stabilitysum} with \cref{eq:penaltystability,eq:penaltybound} completes the proof.
\end{proof}

Combining \cref{lem:unknownbias1,lem:unknownbias2,lem:unknownreg}, using \cref{lem:widthappendix} for unknown transitions, and choosing $\gamma=\eta$ as in \cref{thm:known,thm:unknown} proves the theorems.
The full proofs are given in \cref{app:unknownregret}.

\section{Linear-mixture MDPs}\label{sec:linear-mixture}
In this section, we extend \cref{alg:po} to adversarial MDPs with linear-mixture transitions~\citep{ayoub2020model,zhou2021nearly}.
We use the following feature normalization assumption, following \citet{lzz2024} and \citet{cheng2026policy}.
\begin{assumption}[Linear-mixture transitions]\label{ass:lm}

Let $\phi(s'\mid s,a)\in\R^d$ be known features with $\nrm{\phi(s'\mid s,a)}_2\leq1$, and let $B\geq1$ be a known constant. For each $h=0,\ldots,H-1$, there exists an unknown $\theta_h\in\R^d$ with $\nrm{\theta_h}_2\leq B$ such that $P(s'\mid s,a)=\inpr{\phi(s'\mid s,a),\theta_h}$ for all $(s,a,s')\in\calS_h\times\calA\times\calS_{h+1}$.
\end{assumption}
An MDP with a transition kernel of this form is called a linear-mixture MDP.

Under \cref{ass:lm}, we replace the confidence set of \cref{sec:confidence} with one built from an estimate of $\theta_h$.
Before episode $t$, we compute the following vector least-squares estimator~\citep{lzz2024,cheng2026policy}:
\begin{align}
\Lambda_{t,h}
&=I_d+\sum_{\tau<t}\sum_{s'\in\calS_{h+1}}\phi(s'\mid s_{\tau,h},a_{\tau,h})\phi(s'\mid s_{\tau,h},a_{\tau,h})^\top,
\label{eq:lm-gram}\\
\widehat\theta_{t,h}
&=\Lambda_{t,h}^{-1}\sum_{\tau<t}
\phi(s_{\tau,h+1}\mid s_{\tau,h},a_{\tau,h}).
\label{eq:lm-estimator}
\end{align}
Then, following \citet[Eq.~(4)]{cheng2026policy}, define the confidence radius $\beta_t$ as
\begin{align}\label{eq:lm-radius}
\beta_t
=B+\sqrt{2\log\prn*{\frac{H}{\delta}}
+d\log\prn*{4+\frac{4S(t-1)}{d}}},
\end{align}
the parameter confidence set
\begin{align}\label{eq:lm-ellipsoid}
\Theta_{t,h}
=\set*{\theta\in\R^d:
\nrm{\theta-\widehat\theta_{t,h}}_{\Lambda_{t,h}}\leq\beta_t},
\end{align}
and, for each $h$ and $(s,a)\in\calS_h\times\calA$, the confidence set
\begin{align}
\calP_t(s,a)=\Bigl\{\widetilde P(\cdot\mid s,a)\in\Delta(\calS_{h+1}) \colon \exists\theta\in\Theta_{t,h},\widetilde P(s'\mid s,a)=\inpr{\phi(s'\mid s,a),\theta}\ \forall s'\in\calS_{h+1}\Bigr\},
\label{eq:lm-row-set}
\end{align}
with $\calP_t=\prod_{s,a}\calP_t(s,a)$.
By \citet[Lemma~12]{cheng2026policy}, $P\in\calP_t$ for all $t\in[T]$ with probability at least $1-\delta$.
Since $\theta$ is chosen separately for each $(s,a)$, $\calP_t$ is rectangular, and we use it in \cref{alg:po} in place of \cref{eq:confrow}.
If $\calP_t(s,a)$ is empty for some $(s,a)$, we set $\calP_t(s,a)=\Delta(\calS_{h+1})$.
The following lemma is the counterpart of \cref{lem:widthappendix} for these confidence sets. 
\begin{lemma}[{\citealt[Lemma~16]{cheng2026policy}}]\label{lem:lm-confidence-width}
Suppose that \cref{ass:lm} holds and $P\in\calP_t$ for all $t\in[T]$. Then, with probability at least $1-\delta$,
\begin{align}
\sum_{t=1}^T\sum_{s,a}
\prn*{\overline q_t^{\pi_t}(s,a)-\underline q_t^{\pi_t}(s,a)}\lesssim dS^{3/2}\sqrt{HT}\log\prn*{\frac{ST}{\delta}}.
\end{align}
\end{lemma}
Combining \cref{lem:lm-confidence-width} with the regret analysis in \cref{sec:known}, we obtain the following bound in the same way as \cref{thm:known,thm:unknown}. The proof is given in \cref{app:lm-proof}.
\begin{theorem}\label{thm:lm}
Under \cref{ass:lm}, \cref{alg:po} with the confidence sets \eqref{eq:lm-row-set} and $\gamma=\eta=\sqrt{\frac{H\log(SAT/\delta)}{S(H+A)T}}$ satisfies, with probability at least $1-5\delta$,
\begin{align}
\operatorname{Reg}_T\lesssim\sqrt{HSAT\log\prn*{\frac{SAT}{\delta}}}+dS^{3/2}\sqrt{HT}\log\prn*{\frac{SAT}{\delta}}.
\end{align}
\end{theorem}
As in \cref{thm:unknown}, the second term comes from the occupancy width in \cref{lem:lm-confidence-width}, and it improves the horizon dependence of the policy optimization bound of \citet{cheng2026policy} by a factor of $H$.
The rate matches the best-known bound $\tilO(\sqrt{HSAT}+dS^{3/2}\sqrt{HT})$ up to logarithmic factors, achieved by the occupancy-measure method of \citet{lzz2024}.

\section{Conclusion}\label{sec:conclusion}
We introduced a new policy optimization framework based on regularized $Q$-functions that improves the horizon dependence of regret bounds while retaining local policy updates.
For tabular MDPs with unknown transitions, our algorithm closes the horizon gap between policy optimization and the best available regret bounds achieved by global optimization.
We believe that this framework can also reduce the horizon dependence of policy optimization under delayed feedback and in best-of-both-worlds settings, and leave these extensions for future work.

\section*{AI use statement}
We used GPT-6-Astra, GPT-6.1-Sol, and Claude Opus 5.5 to assist with proof development and checking, literature searches, and editing.
We reviewed all AI-assisted work and take responsibility for the paper.

\bibliographystyle{plainnat}
\bibliography{ref}

\clearpage
\appendix
\crefalias{subsection}{appendix}

\section{Proofs deferred from \NoCaseChange{\protect\cref{sec:algorithm}}}\label{app:unknownregret}
We prove \cref{thm:known,thm:unknown} and the supporting \cref{lem:unknownbias1,lem:unknownbias2}.

\begin{proof}[Proof of \cref{lem:unknownbias1}]
Let $Y_t=\inpr{\underline q_t^{\pi_t},\hat\ell_t}$. Decompose
\begin{align}
\textup{\textsc{Bias}}_1
&=\sum_{t=1}^T\prn*{\inpr{q^{\pi_t},\ell_t}
-\inpr*{\underline q_t^{\pi_t},\E_t\brk*{\hat\ell_t}}}+\sum_{t=1}^T\prn*{\E_t\brk*{Y_t}-Y_t}.
\end{align}
For the first term,
\begin{align}
q^{\pi_t}(s,a)\ell_t(s,a)-\underline q_t^{\pi_t}(s,a)
\E_t\brk*{\hat\ell_t(s,a)}
&=\frac{q^{\pi_t}(s,a)\ell_t(s,a)}
{\overline q_t^{\pi_t}(s,a)+\gamma}\prn*{
\overline q_t^{\pi_t}(s,a)-\underline q_t^{\pi_t}(s,a)+\gamma}\\
&\leq\overline q_t^{\pi_t}(s,a)
-\underline q_t^{\pi_t}(s,a)+\gamma,
\end{align}
where the inequality follows from $q^{\pi_t}(s,a)\leq\overline q_t^{\pi_t}(s,a)$.
For the second term,
\begin{align}
0\leq Y_t=\sum_{s,a}\I_t(s,a)
\frac{\underline q_t^{\pi_t}(s,a)\ell_t(s,a)}{\overline q_t^{\pi_t}(s,a)+\gamma}\leq\sum_{s,a}\I_t(s,a)=H.
\end{align}
Hence $\abs*{\E_t\brk*{Y_t}-Y_t}\leq H$. Applying \cref{lem:hoeffding} to $X_t=\E_t\brk*{Y_t}-Y_t$ gives, with probability at least $1-\delta$,
\begin{align}
\sum_{t=1}^T\prn*{\E_t\brk*{Y_t}-Y_t}
\leq H\sqrt{2T\log\prn*{\frac1\delta}}.
\end{align}
Combining the two bounds completes the proof.
\end{proof}

\begin{proof}[Proof of \cref{lem:unknownbias2}]
By the proof of \citet[Lemma~14]{jin2020learning}, with probability at least $1-\delta$,
\begin{align}
\sum_{t=1}^T\prn*{\hat\ell_t(s,a)-\ell_t(s,a)}\leq\frac{1}{2\gamma}\log\prn*{\frac{SA}{\delta}}
\end{align}
holds for all $(s,a)$. Then,
\begin{align}
\textup{\textsc{Bias}}_2=\sum_{s,a}q^{\pio}(s,a)\sum_{t=1}^T
\prn*{\hat\ell_t(s,a)-\ell_t(s,a)}
\leq\frac{H}{2\gamma}\log\prn*{\frac{SA}{\delta}},
\end{align}
which completes the proof.
\end{proof}

\begin{proof}[Proof of \cref{thm:known,thm:unknown}]
If $P\in\calP_t$ for all $t\in[T]$, then the regret decomposition in \cref{eq:unknowndecomp} and \cref{lem:unknownbias1,lem:unknownbias2,lem:unknownreg} give, with probability at least $1-3\delta$,
\begin{align}\label{eq:regcombined}
\operatorname{Reg}_T
&\lesssim\frac{H\log A}{\eta}+\gamma SAT+\eta S(H+A)T+H\sqrt{T\log\prn*{\frac{SA}{\delta}}}
+\frac{\eta H^2+H}{\gamma}\log\prn*{\frac{SA}{\delta}}\\
&\quad+\sum_{t=1}^T\sum_{s,a}\prn*{\overline q_t^{\pi_t}(s,a)
-\underline q_t^{\pi_t}(s,a)}. \label{eq:common_bound}
\end{align}
For known transitions, $\calP_t=\{P\}$ and the last term is zero. Choosing $\gamma=\eta=\sqrt{\frac{H\log(SA/\delta)}{S(H+A)T}}$ proves \cref{thm:known}.

For unknown transitions, \citet[Lemma~2]{jin2020learning} gives $P\in\calP_t$ for all $t\in[T]$ with probability at least $1-4\delta$. Applying \cref{lem:widthappendix} to the last term in \cref{eq:common_bound} and choosing $\gamma=\eta=\sqrt{\frac{H\log(SAT/\delta)}{S(H+A)T}}$ proves \cref{thm:unknown}.
\end{proof}
\section{Proofs deferred from \NoCaseChange{\protect\cref{sec:linear-mixture}}}\label{app:lm-proof}
\begin{proof}[Proof of \cref{thm:lm}]
By \cref{sec:linear-mixture}, $P\in\calP_t$ for all $t\in[T]$ with probability at least $1-\delta$. Then \cref{eq:common_bound}, \cref{lem:lm-confidence-width}, and the choice of $\gamma=\eta$ give, with probability at least $1-5\delta$,
\begin{align}
\operatorname{Reg}_T
&\lesssim\sqrt{HS(H+A)T\log\prn*{\frac{SAT}{\delta}}}+H^2\log\prn*{\frac{SAT}{\delta}}+dS^{3/2}\sqrt{HT}\log\prn*{\frac{ST}{\delta}}\\
&\lesssim\sqrt{HSAT\log\prn*{\frac{SAT}{\delta}}}+dS^{3/2}\sqrt{HT}\log\prn*{\frac{SAT}{\delta}},
\end{align}
where the last inequality follows from $H\sqrt{ST}+H^2\lesssim dS^{3/2}\sqrt{HT}$, which holds since $H\leq S$. This completes the proof.
\end{proof}

\section{Concentration inequalities}\label{app:concentration}
We provide the concentration inequalities used in the regret proofs.

\begin{lemma}[Azuma--Hoeffding inequality]\label{lem:hoeffding}
Let $X_t$ be $\calF_{t+1}$-measurable with $\E_t\brk*{X_t}=0$ and $\abs{X_t}\leq R$ almost surely for every $t$. Then for any $\delta\in(0,1)$, with probability at least $1-\delta$,
\begin{align}
\sum_{t=1}^TX_t\leq R\sqrt{2T\log\prn*{\frac{1}{\delta}}}.
\end{align}
\end{lemma}

\begin{lemma}[{\citealt[Lemma~11]{jin2020learning}}]\label{lem:jointIX}
For any sequence of $\calF_t$-measurable functions $\alpha_t:\calS\times\calA\to[0,2\gamma]$, with probability at least $1-\delta$,
\begin{align}\label{eq:jointIXlayer}
\sum_{t=1}^T\sum_{s,a}\alpha_t(s,a)\prn*{
\hat\ell_t(s,a)
-\frac{q^{\pi_t}(s,a)}{\overline q_t^{\pi_t}(s,a)}
\ell_t(s,a)}
\leq H\log\prn*{\frac H\delta}.
\end{align}
\end{lemma}

\end{document}